\documentclass{article}

    \usepackage[preprint]{neurips_2024}

\usepackage[utf8]{inputenc} % allow utf-8 input
\usepackage[T1]{fontenc}    % use 8-bit T1 fonts
\usepackage{hyperref}       % hyperlinks
\usepackage{url}            % simple URL typesetting
\usepackage{booktabs}       % professional-quality tables
\usepackage{amsfonts}       % blackboard math symbols
\usepackage{nicefrac}       % compact symbols for 1/2, etc.
\usepackage{microtype}      % microtypography
\usepackage{xcolor}         % colors

\usepackage{natbib}

\usepackage{graphicx}
\usepackage{multirow}
\usepackage{booktabs}
\usepackage{subcaption}
\usepackage{threeparttable, array, float}
\usepackage{colortbl}

\usepackage{textcomp}

\usepackage{xcolor}

\usepackage{hyperref}
\definecolor{mydarkblue}{rgb}{0,0.08,0.45}
\definecolor{blueberry}{RGB}{4,51,255}
\hypersetup{ %
    pdftitle={},
    pdfsubject={},
    pdfkeywords={},
    pdfborder=0 0 0,
    pdfpagemode=UseNone,
    colorlinks=true,
    linkcolor=mydarkblue,
    citecolor=mydarkblue,
    filecolor=mydarkblue,
    urlcolor=mydarkblue,
}

\usepackage{amsmath,amsfonts,amssymb}
\usepackage{bm, bbm}
\usepackage{mathtools}
\usepackage{nicefrac}

\makeatletter
\newcommand\addstarred[1]{%
    \expandafter\let\csname\string#1@nostar\endcsname#1%
    \edef#1{\noexpand\@ifstar\expandafter\noexpand\csname\string#1@star\endcsname\expandafter\noexpand\csname\string#1@nostar\endcsname}%
    \expandafter\newcommand\csname\string#1@star\endcsname%
}
\makeatother
\newcommand{\1}[1]{\mathbbm{1}{\{#1\}}}
\addstarred\1[1]{\mathbbm{1}{\left\{#1\right\}}}

\DeclarePairedDelimiter\ceil{\lceil}{\rceil}
\DeclarePairedDelimiter\floor{\lfloor}{\rfloor}
\DeclarePairedDelimiter\abs{\lvert}{\rvert}

\DeclareMathOperator*{\argmax}{arg\,max}

\renewcommand{\ge}{\geqslant}
\renewcommand{\le}{\leqslant}

\newcommand{\type}[1]{Type-\uppercase\expandafter{\romannumeral#1}}

\usepackage{amsthm}
\newtheorem{theorem}{Theorem}
\newtheorem{lemma}[theorem]{Lemma}
\newtheorem{proposition}[theorem]{Proposition}

\newtheorem{remark}[theorem]{Remark}

\usepackage{algorithm}
\usepackage[noend]{algpseudocode}

\let\oldComment=\Comment
\renewcommand{\Comment}[1]{\oldComment{\texttt{#1}}}
\algnewcommand{\LeftComment}[1]{\Statex $\triangleright$ \texttt{#1}}
\algnewcommand{\RightComment}[1]{\Statex \leavevmode\hfill$\triangleright$ \texttt{#1}}

\algnewcommand\algorithmicinput{\textbf{Input:}}
\algnewcommand\Input{\item[\algorithmicinput]}%

\algnewcommand\algorithmicoutput{\textbf{Output:}}
\algnewcommand\Output{\item[\algorithmicoutput]}%

\algnewcommand\algorithmicinitial{\textbf{Initialize:}}
\algnewcommand\Initial{\item[\algorithmicinitial]}%

\usepackage{xspace}
\usepackage{comment}

\usepackage{fontawesome5}

\usepackage[colorinlistoftodos,prependcaption,textsize=tiny,textwidth=2cm,color=green]{todonotes}
\usepackage{thmtools}

\usepackage{wrapfig}

\newcommand{\sewr}{{\normalfont\texttt{SE-WR}}\xspace}
\newcommand{\sewrst}{{\normalfont\texttt{SE-WR-Stop}}\xspace}
\newcommand{\pewr}{{\normalfont\texttt{PE-WR}}\xspace}
\newcommand{\mssewr}{{\normalfont\texttt{MS-SE-WR}}\xspace}
\newcommand{\mab}{{\normalfont\texttt{MAB}}\xspace}
\newcommand{\CInput}{C_{\text{Input}}}

\title{Stochastic Bandits Robust to Adversarial Attacks}

\author{%
  Xuchuang Wang \\
  CICS, UMass Amherst\\
    Amherst, MA 01003 \\
  \texttt{xuchuangwang@cs.umass.edu} \\
  \And
  Jinhang Zuo \\
  CS, CityU \\
  Kowloon, Hong Kong\\
  \texttt{jinhangzuo@gmail.com} \\
  \And
  Xutong Liu \\
  CSE, CUHK \\
  New Territories, Hong Kong \\
  \texttt{liuxt@cse.cuhk.edu.hk} \\
  \AND
  John C.S. Lui \\
  CSE, CUHK \\
  New Territories, Hong Kong \\
  \texttt{cslui@cse.cuhk.edu.hk} \\
  \And
    Mohammad Hajiesmaili\\
    CICS, UMass Amherst\\
    Amherst, MA 01003 \\
  \texttt{hajiesmaili@cs.umass.edu} \\
}

\begin{document}

\maketitle

\begin{abstract}
    This paper investigates stochastic multi-armed bandit algorithms that are robust to adversarial attacks, where an attacker can first observe the learner's action and \emph{then} alter their reward observation.
    We study two cases of this model, with or without the knowledge of an attack budget $C$, defined as an upper bound of the summation of the difference between the actual and altered rewards. For both cases, we devise two types of algorithms with regret bounds having additive or multiplicative $C$ dependence terms.
    For the known attack budget case, we prove our algorithms achieve the regret bound of ${O}((K/\Delta)\log T + KC)$ and $\tilde{O}(\sqrt{KTC})$ for the additive and multiplicative $C$ terms, respectively, where $K$ is the number of arms, $T$ is the time horizon, $\Delta$ is the gap between the expected rewards of the optimal arm and the second-best arm, and \(\tilde{O}\) hides the logarithmic factors.
    For the unknown case, we prove our algorithms achieve the regret bound of $\tilde{O}(\sqrt{KT} + KC^2)$ and $\tilde{O}(KC\sqrt{T})$ for the additive and multiplicative $C$ terms, respectively.
    In addition to these upper bound results, we provide several lower bounds showing the tightness of our bounds and the optimality of our algorithms.
    These results delineate an intrinsic separation between the bandits with attacks and corruption models~\citep{lykouris2018stochastic}.
    % \mo{either remove that last sentence or motivate earlier in abstract the diffs.} \xw{how about we add a reference to the corruption? It is difficult to discuss the difference in the abstract.}
\end{abstract}

% \begin{itemize}
%     \item[\todocircle] Try to write the model, introduction, and related works in one section.
%           % \item[\todocircle] Lower Bound results writing.
% \end{itemize}

% !TeX root = ../robust-to-attack.tex
\section{Introduction}

% \todo{remove the weak adversary and adversarial bandits. Among nonoblivious adversaries, }
Online learning literature~\citep[\S7.1.2]{borodin2005online}
usually considers two types of non-oblivious adversary models: the medium adversary and the strong adversary.\footnote{There is a third adversary model, called the weak adversary, which is oblivious. The medium adversary is also called the adaptive-online adversary, while the strong adversary is called the adaptive-offline adversary.} The medium adversary chooses the next instance \emph{before} observing the learner's actions, while the strong adversary chooses instances \emph{after} observing the learner's actions.
% The strong adversary is also known as the strong adversary.
When it comes to the multi-armed bandits (\mab) learning with an adversary, the medium adversary corresponds to the bandits with corruption~\citep{lykouris2018stochastic}, and the strong adversary corresponds to adversarial attacks on bandits~\citep{jun2018adversarial}. Bandit algorithms robust to corruption are developed for the medium adversary models in bandits literature, e.g.,~\citet{auer2002nonstochastic,audibert2010regret,lykouris2018stochastic,gupta2019better}. However, for the strong adversary model, i.e., the adversarial attack model on \mab, most previous efforts focus on devising attacking policies to mislead the learner to pull a suboptimal arm and thus suffer a linear regret, e.g., \citet{jun2018adversarial,liu2019data,zuo2024near}. Developing robust algorithms for the adversarial attack model and achieving regret with benign dependence on the total attack is still largely unexplored (for a detailed discussion, see the end of \S\ref{sec:related-works}).

\textbf{Stochastic \mab with adversarial attacks}
In this paper, we study the stochastic \mab under adversarial attacks and develop robust algorithms whose regret bounds degrade gracefully in the presence of such attacks.
Denote by \(K\in\mathbb{N}^+\) as the number of arms in \mab, and each arm \(k\) is associated with a reward random variable \(X_k\) with an unknown mean \(\mu_k\). Denote by \(k^*\) as the arm index with the highest mean reward, and \(\Delta_k \coloneqq \mu_{k^*} - \mu_k\) the reward gap between the optimal arm \(k^*\) and any suboptimal arm \(k\). The learner aims to minimize the \emph{regret}, defined as the difference between the highest total rewards of pulling a single arm and the accumulated rewards of the concern algorithm. In each time slot, the learner first pulls an arm. Then, the adversary observes the pulled arm and its realized reward and chooses an attacked reward for the learner to observe. Denote by \(C\) as the total amount of \emph{attacks} that the adversary uses to alter the raw rewards to the attacked rewards over all rounds.
We present the formal model definitions in \S\ref{sec:model}.

\textbf{Contributions}
This paper proposes robust algorithms with tight regret bounds for the adversarial attack model on \mab. To achieve that, we first develop robust algorithms with the known attack budget as an intermediate step towards developing robust algorithms for the unknown attack budget case. Later on, we call the former the \emph{known attack} case and the latter the \emph{unknown attack} case.

\noindent\textit{Known attack budget case.} For the known attack case (\S\ref{sec:algorithm-known}), we first investigate the successive elimination with wider confidence radius (\sewr) algorithm, which was initially proposed for bandits with corruptions~\citep{lykouris2018stochastic}.
In~\S\ref{subsec:algorithm-known-additive}, we show \sewr can be applied to the bandits with known adversarial attack budget
and prove that this algorithm enjoys a \textit{tighter}
regret upper bound \(O( \sum_{k\neq k^*} {\log T}/{\Delta_k} + KC )\) (under both attack and corruption models) than previous analysis~\citep{lykouris2018stochastic} under the corruption model.
% , where \(\Delta_k \coloneqq \mu_{k^*} - \mu_k\) is the reward gap between the optimal arm \(k^*\) and the suboptimal arm \(k\).
This improvement removes the original additive term \(\sum_{k\neq k^*} {C}/{\Delta_k}\)'s dependence of \(\Delta_k\) and reduces it to \(KC\).
To achieve gap-independent regret bounds in~\S\ref{subsec:algorithm-known-multiplicative}, we propose two stopping conditions for the successive elimination to modify the \sewr algorithm to two \sewrst algorithms (see \textcircled{\small{1}} in Figure~\ref{fig:algorithm-design}).
We prove that the regrets of these two modified algorithms are \(O(\sqrt{KT\log T} + KC)\) and \(O(\sqrt{KT(\log T + C)})\) respectively.
% These algorithms serve as building blocks for the unknown attack case, and deriving above tighter regret bounds is crucial for achieving tight bounds for the later unknown attack case.

\noindent\textit{Unknown attack budget case.} To address with \mab with unknown attack case (\S\ref{sec:algorithm-unknown}), we use algorithms proposed for known attacks as building blocks to devise algorithms for unknown attacks. We consider two types of algorithmic techniques.
In~\S\ref{subsec:algorithm-unknown-additive}, we apply a multi-phase idea to extend the \sewr algorithm to a phase elimination algorithm (\pewr, see \textcircled{\small{2}} in Figure~\ref{fig:algorithm-design}). The \pewr algorithm utilizes a multi-phase structure---each phase with carefully chosen attack budgets and phase lengths---to defend against unknown attacks. We show that the regret of \pewr is upper bounded by \(O(\sqrt{KT} + KC^2)\). In~\S\ref{subsec:algorithm-unknown-multiplicative}, we apply a model selection technique, which treats the unknown attack \(C\) as a parameter in model selection (see \textcircled{\small{3}} in Figure~\ref{fig:algorithm-design}). Specifically, we consider \(\log_2 T\) base instances of \sewrst with different input attack budgets and then use the model selection technique to ``corral'' these base instances, composing the model-selection \sewrst (called \mssewr) algorithm. We upper bound the regret of \mssewr by \(\tilde{O}(\sqrt{KC}T^{\frac 2 3})\) or \(\tilde{O}(KC\sqrt{T})\) depending on the choice of the model selection algorithm.
Figure~\ref{fig:algorithm-design} summarizes the relations of algorithm design, and Figure~\ref{fig:unknown-C-comparison} shows that which type of algorithm performs better, those with additive bounds or those with multiplicative bounds, varies depending on the total attacks.
% compare the order of their regret upper bounds with respect to the amount of attacks in 

% \todo{explain why Figure~\ref{tab:attack-corruption-comparison} also in caption.}

\begin{figure}[tb]
    \centering
    \begin{minipage}{.58\textwidth}
        \centering
        \includegraphics[width=\linewidth]{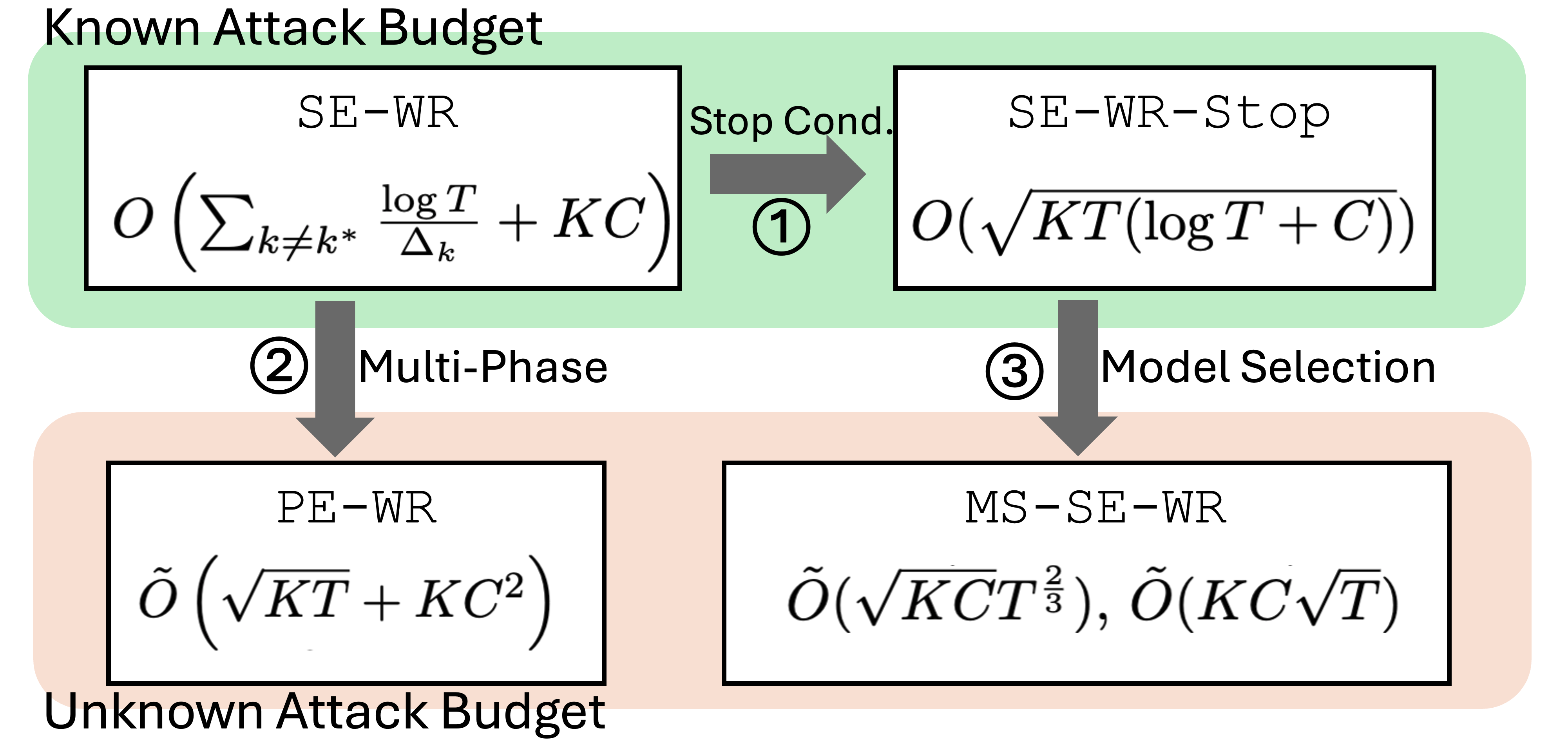}
        \caption{Algorithm design overview: only \sewr has a gap-dependent bound; others are all gap-independent.}
        \label{fig:algorithm-design}
    \end{minipage}%
    \hfill
    \begin{minipage}{.375\textwidth}
        \centering
        \includegraphics[width=\linewidth]{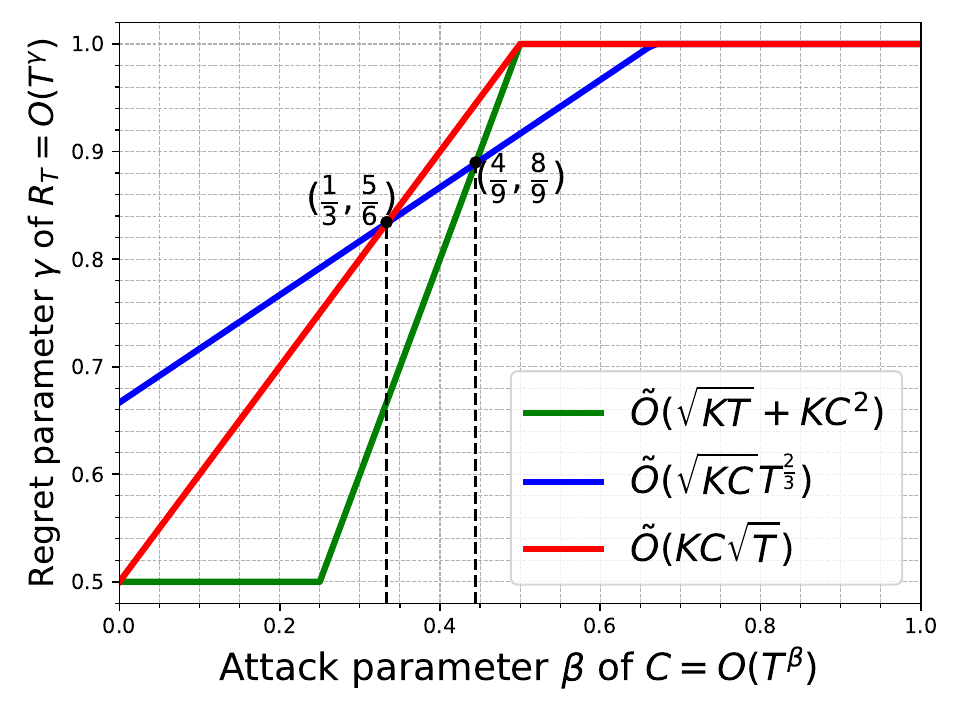}
        \caption{Comparison of unknown \(C\) regrets 
        % x-axis varies \(\beta\), controlling total attacks \(T^\beta\), and y-axis is the exponent of \(T^\gamma\) regret bounds 
        (see Remark~\ref{remark:regret-bounds-comparison} for detail)
        }
        \label{fig:unknown-C-comparison}
    \end{minipage}%
\end{figure}
% \todo{If have time, try to re-plot the figure via TikZ.}

\noindent\textit{Lower bounds.} In \S\ref{sec:lower-bound}, we also provide lower bound results to show the tightness of our upper bounds. We first show a general \(\Omega(KC)\) lower bound for any algorithm under attack. Based on this, we propose the gap-dependent lower bound \(\Omega( \sum_{k\neq k^*} {\log T}/{\Delta_k} + KC )\) (informal expression) and gap-independent lower bound \(\Omega(\sqrt{KT} + KC)\). Our refined upper bound of \sewr matches the gap-dependent lower bound up to some prefactors, and one of our proposed \sewrst algorithms matches the gap-independent lower bound up to some logarithmic factors.
To show the tightness of our upper bounds for the unknown attack algorithms, we derive lower bounds, \(\Omega(T^\alpha + C^{\frac{1}{\alpha}})\) and \(\Omega(C^{\frac 1 \alpha - 1}T^\alpha)\) with \(\alpha \in [\frac 1 2, 1)\), for two classes of algorithms whose regret upper bounds follow certain additive and multiplicative forms respectively.
With \(\alpha = \frac{1}{2}\), these two lower bounds, becoming \(\Omega(\sqrt{T} + C^2)\) and \(\Omega(C \sqrt{T})\), match the upper bounds of \pewr and \mssewr in terms of \(T\) and \(C\).
We summarize our analysis results in Table~\ref{tab:overview}.

\begin{table}[tp]
    \centering
    \caption{
        Results overview: upper bounds with \(^\dagger\) are for pseudo regret in expectation, while all other upper bounds are for realized regret with high probability;
        lower bounds with \(^\ddagger\) are only for special classes of algorithms (see Propositions~\ref{prop:additive-lower-bound} and~\ref{prop:multiplicative-lower-bound}), and the lower bound with \(^{\dagger\dagger}\) holds when \(T\to \infty\).
        We report both additive and multiplicative bounds because which one is better could depend on the value of \(C\) (see Figure~\ref{fig:unknown-C-comparison} for unknown \(C\) case and the detailed discussion in Remark~\ref{remark:regret-bounds-comparison}).
    }\label{tab:overview}
    \resizebox{\columnwidth}{!}{
        \begin{tabular}{cccc}
            \toprule
            \textbf{Regret Bounds}
             & \textbf{Known Attack \(C\)} (\S\ref{sec:algorithm-known})
             & \textbf{Unknown Attack \(C\)} (\S\ref{sec:algorithm-unknown})
            \\
            \midrule
            \textbf{Additive \(C\)}
             & \(O\left( \sum_{k\neq k^*} \frac{\log T}{\Delta_k} + KC \right),\,O(\sqrt{KT\log T} + KC)\)
             & \(\tilde{O}\left( \sqrt{KT} + KC^2 \right)\)
            \\
            \textbf{Multiplicative \(C\)}
             & \(O(\sqrt{KT(\log T + C)})\)
             & \(\tilde{O}(\sqrt{KC}T^{\frac 2 3})^\dagger,\,\tilde{O}(KC\sqrt{T})^\dagger\)
            \\
            % \hline
            \textbf{Lower Bounds} (\S\ref{sec:lower-bound})
             & \(\Omega\left( \sum_{k\neq k^*} \frac{\log T}{\Delta_k} + KC \right)^{\dagger\dagger},\Omega(\sqrt{KT} + KC)\)
             & \(\Omega(T^\alpha + C^{\frac{1}{\alpha}})^\ddagger, \Omega(C^{\frac 1 \alpha - 1}T^\alpha)^\ddagger\)
            \\
            \bottomrule
        \end{tabular}}
\end{table}

% \todo{may further explain this point to make the difference between attack and corruption models more clear.}
% \begin{itemize}
%     \item[\check] Explain why~\cite{lykouris2018stochastic,gupta2019better} did not work on the adversarial attack case. \begin{itemize}
%               \item[\todocircle] I may further explain this point to make the difference between attack and corruption models more clear.
%           \end{itemize}

%     \item[\check] Use \citet{zuo2024near} as our starting point.

%     \item[\todocircle] Carefully differentiate the related works on robust to attack papers~\citep{bogunovic2020corruption,kang2024robust}.

%     \item[\check] slightly mention the best of both world literature~\citep{zimmert2021tsallis}.
% \end{itemize}

\noindent\textit{Result separation between corruption and attack.}
With the results in this paper, we demonstrate a clear \textit{separation} between corruption (medium adversary) and attack (strong adversary) models in terms of the additive regret bounds.
Denote by \(C'\) the total corruption to distinguish from the total attack \(C\), and we discuss the model different in more detail in \S\ref{subsec:model-difference}.
In the case of known corruption/attack budget, while both achieving the gap-dependent regret bounds \(\sum_{k\neq k^*} (1/\Delta_k) \log T\) in \(T\)-dependent terms, we show that the regret due to attack is \(\Theta(KC)\) which has a factor of \(K\) worse than the additional \(\Theta(C')\) regret due to corruption.
That implies, with the same amount of budget, an attacker can produce \(K\) times larger impact on regret than that under the corruption model.
For unknown budget cases,
while bandits with corruptions have the regret bounds of \(O\left( \sum_{k\neq k^*} \frac{\log^2(KT/\delta)}{\Delta_k} + C'\right)\)~\citep{gupta2019better,liu2021cooperative}, it is impossible for attack case to achieve a regret upper bound in the form of \(O\left( \text{ploylog}(T) + C^\alpha\right)\) for any \(\alpha > 0\) (contradicts with Theorem~\ref{thm:linear-regret-condition}). Instead, we show that the regret with unknown attacks can attain \(\tilde{O}(\sqrt{KT} + KC^2)\) where the \(C^2\) additive term is tight. Even if ignoring the worse order of the first \(T\)-related terms, the second regret term due to attack \(O(C^2)\) is still a factor of \(C\) worse than the additional \(O(C')\) regret due to corruption. That is, an attack with \(O(\sqrt{T})\) budget is enough to make the algorithm suffer a linear regret in the attack model,\footnote{This statement is for the algorithm with \(\tilde{O}(\sqrt{KT} + KC^2)\) mentioned above. Our lower bound results do not exclude the possible existence of other algorithms with upper bounds \(\tilde{O}(T^\alpha + C^{\frac{1}{\alpha}})\) for \(\alpha \in [\frac 1 2, 1)\), in which case a sublinear \(O(T^\alpha)\) budget can make the algorithm suffer linear regret.} while one needs a linear \(O(T)\) corruption budget to achieve the same effect in the corruption case.

% \begin{remark}[Compare with corruption]
%     We note that this
%     the procedure of the attack and corruption bounds are fundamentally different.

%     \todocircle~improve an explanation on the key difference. The current one is far from satisfactory.

%     One potential misunderstanding is that the attack and corruption only differ slightly. For example, one could argue that given any attack policy, that is, at time slot \(t\), if arm \(k\) is pulled, the adversary alters the reward realization to another value, then, in the corruption case, before revealing which arm is pulled, the adversary can choose the same reward alteration for all arms, which has the same effect of only altering the pulled arm in the attack case. However, this is not the case. Because the adversary in the attack process can save a lot of cost by only altering the pulled arm, while in the corruption case, the adversary has to alter all arms to achieve the same effect. The latter could result in much higher cost than the former. This is the key difference between the attack and corruption process.
% \end{remark}

\textbf{Extended literature review}\label{sec:related-works}
\noindent\textit{Robustness against corruption.}
\citet{lykouris2018stochastic} propose the bandits with corruption model and devised algorithms with \(O((K\log T + KC')/\Delta)\) regret for known corruption \(C'\) case and algorithms with \({O}(C'K^2\log T/ \Delta)\) regret for unknown \(C'\) case, where \(\Delta\) denotes the smallest non-zero reward gap.
% \footnote{In order to differentiate from the notation \(C\) for the amount of attack, we use \(C'\) to denote the total amount of corruption.}
Later, \citet{gupta2019better} study the same model and propose an algorithm that improves the multiplicative \(C'\) dependence of regret for unknown \(C'\) case to additive \(C'\), i.e., \({O}(K\log^2 T/\Delta + KC')\).
However, the robust design of both algorithms above relies on randomly pulling arms, which is not feasible in our attack setting.
% They showed that the regret of the learner is at most \(O(\sqrt{KT})\) when the corruption is bounded by \(C\). 
Besides bandits robust to corruption, there is another line of works on best-of-both-world bandit algorithm design, e.g.,~\citet{seldin2017improved,wei2018more,zimmert2021tsallis}, where their algorithms can be applied to the bandit with corruption but do not work in the adversarial attack model either.

\noindent\textit{Attacks on bandits.}
Attack policy design for \mab algorithms has been studied by~\citet{jun2018adversarial,liu2019data,xu2021observation,zuo2024adversarial,zuo2024near} and many others.
They aim to design attacking policies to mislead the learner, typically with logarithmic regrets when there is no attack, to pull a suboptimal arm and thus suffer a linear regret. This is the opposite of our objective for designing robust algorithms to achieve sublinear regret under any attacks.

\noindent\textit{Robustness against attacks.}
The only existing result on robust algorithms for \mab under attacks is by~\citet[Section 6]{zuo2024near}, where the author uses competitive ratio as the objective---often used when one cannot achieve sublinear regret, instead of the finer-grained regret metric studied in this paper.
Apart from the \mab model, there are works studying robust algorithms on structured bandits under attacks, e.g., linear bandits~\citep{bogunovic2020corruption,he2022nearly} and Lipchitz bandits~\citep{kang2024robust}. Although their linear/Lipchitz bandits results can be reduced to \mab, this reduction does not provide algorithms with tight regrets as in Table~\ref{tab:overview} of this paper. We will elaborate on this in \S\ref{sec:algorithm-unknown}.

% while achieving the near-optimal gap-independence regret bounds \(O(\sqrt{KT})\), we show that the regret due to attack is \(\Theta(C^2)\) which has a factor of \(C\) worse than the additional \(\Theta(C)\) regret due to corruption~\citep{gupta2019better,liu2021cooperative}.

% We summarize this comparison in Table~\ref{tab:attack-corruption-comparison}.

% \begin{table}[H]
%     \centering
%     \caption{Bound comparison between attack and corruption (the second additive terms are all tight)}
%     \label{tab:attack-corruption-comparison}
%     \begin{tabular}{ccc}
%         \toprule
%         Regret upper bounds
%          & Known \(C\) or \(C'\)
%          & Unknown \(C\) or \(C'\)
%         \\
%         \midrule
%         Attack (ours)
%          & \(O\left( \sum_{k\neq k^*} \frac{\log(KT/\delta)}{\Delta_k} + KC\right)\)
%          & \(\tilde{O}\left( \sqrt{KT} + KC^2 \right)\)
%         \\
%         Corruption
%          & \multicolumn{2}{c}{\(O\left( \sum_{k\neq k^*} \frac{\log^2(KT/\delta)}{\Delta_k} + C'\right)\)}
%         %  & \(\tilde{O}\left( \sqrt{KT} + C' \right)\)
%         \\
%         \bottomrule
%     \end{tabular}
% \end{table}
\section{Model: \mab with Adversarial Attacks}\label{sec:model}

We consider a \mab model with \(K\in\mathbb{N}^+\) arms. Each arm \(k \in \mathcal{K}\coloneqq \{1,2,\dots, K\}\) is associated with a reward random variable \(X_k\in[0,1]\) with an unknown mean \(\mu_k\). Denote the unique arm with highest reward mean as \(k^*\), i.e., \(k^* = \argmax_{k\in\mathcal{K}} \mu_k\).
We consider \(T\in\mathbb{N}^+\) decision rounds.
At each round \(t\in\mathcal{T}\coloneqq\{1,2,\dots, T\}\), the learner selects an arm \(I_t\in\mathcal{K}\), and the arm generates a reward realization \(X_{I_t, t}\) (not disclosed to the learner yet). The adversary observes the pulled arm \(I_t\) and its realized reward \(X_{I_t, t}\) and then chooses an attacked reward \(\tilde{X}_{I_t, t}\) for the learner to observe.
The total attack is the sum of the absolute differences between the raw rewards and the attacked rewards over all rounds, i.e., \(C \coloneqq \sum_{t=1}^T \abs{X_{I_t, t} - \tilde{X}_{I_t, t}}\).
We summarize the decision process in Procedure~\ref{proc:decision-under-attack}.

% \mo{do you need pseudo-code, esp in intro? also do you want to re-emphasize that attacked value is still in (0,1)?} \xw{I try to follow the writing flow of a theoreical work, in which case the model can be introduced in introduction, for example~\citet{gupta2019better}. I agree that put the pseudo-code in introduction may go too far. If don't put it here, we can consider move it to appendix~\ref{apx:corruption}.} \mo{let's move pseudo-code to appendix, or have an extended version of model in section 2. see Gupta paper, a short version of model is in intro (whatever is essential) but the detailed model is in section 2 including the def of regret and in our case comparison with corruption vs. attack could go there.}

\floatname{algorithm}{Procedure}
\begin{algorithm}[tb]
    \caption{Decision under Attacks}\label{proc:decision-under-attack}
    \begin{algorithmic}[1]
        \For {each round \(t=1,2,\dots,T\)}
        \State The learner pulls an arm $I_t$
        \State A stochastic reward \(X_{I_t, t}\) is drawn from this arm \(I_t\)
        \State The adversary observes the pulled arm $I_t$ and its realized reward $X_{I_t, t}$
        \State The adversary chooses an attacked reward $\tilde{X}_{I_t, t}$
        \State The learner only observes the attacked reward $\tilde{X}_{I_t, t}$, and not $X_{I_t, t}$
        \EndFor
    \end{algorithmic}
\end{algorithm}
\floatname{algorithm}{Algorithm}

\textbf{Regret objective}
The adversary aims to maximize the learner's regret, while the learner aims to minimize the regret.
We define the (realized) regret as the difference between the highest total realized \emph{raw} rewards of a single arm and the accumulated \emph{raw} rewards of the learning algorithm as follows,
\begin{equation}
    \label{eq:realized-regret}
    R_T \coloneqq \max_{k\in\mathcal{K}} \sum\nolimits_{t\in\mathcal{T}} ({X}_{k,t} -  {X}_{I_t, t}).
\end{equation}
% where \(I_t\) is the pulled arm at time slot \(t\).
Besides the realized regret in~\eqref{eq:realized-regret}, another common regret definition is the pseudo-regret, defined as
\begin{equation}
    \label{eq:pseudo-regret}
    \bar{R}_T \coloneqq \max_{k\in\mathcal{K}} \mathbb{E}\left[
        \sum\nolimits_{t\in\mathcal{T}} ({X}_{k,t} - {X}_{I_t, t})
        \right]
    = \mathbb{E}\left[ T \mu_{k^*} - \sum\nolimits_{t\in\mathcal{T}}\mu_{I_t} \right].
\end{equation}
By Jensen's inequality, we have \(
\bar{R}_T \le \mathbb{E}[R_T],
\)
implying that the pseudo-regret is a weaker metric than the realized regret. This paper's results are primarily on realized regret, but we also provide results on the pseudo-regret when necessary.

% \mo{I feel that there is too much details in introduction that makes it hard to quickly understand the contribution of the paper, we should discuss pros and cons of moving 1.1 and first paragraph of 1.3 to a new section 2. }

\textbf{Comparison between corruption and attack models}\label{subsec:model-difference}
% \textbf{Model comparison}
Different from the above attack model (strong adversary), there is prior literature on robust bandit algorithms against corruption (medium adversary)~\citep{lykouris2018stochastic,gupta2019better},
where the adversary chooses the corruption rewards \emph{before} observing the learner's actions (see Appendix~\ref{apx:corruption} for more details), and the total amount of corruption is defined as \({
        C' \coloneqq \sum_{t=1}^T \max_k \abs{X_t(k) - \tilde{X}_{k,t}}.}\)
While comparing the definition of corruption \(C'\) with attack \(C\) may lead to the impression that \(C \le C'\) and both models only differ slightly, we emphasize that the two models are fundamentally different due to the sequence of the adversary attacking (corrupting) rewards after (before) observing the learner's actions.

To highlight the crucial difference between the two sequences, we illustrate that, to generate the same degree of impact on a bandit algorithm, the required budget under the corruption model can be much larger than that under the attack model.
For example, one naive attack policy, forcing the learner to suffer linear regrets, could be that whenever the algorithm pulls the optimal arm \(k^*\), the adversary alters the reward realization to be smaller than the mean of the best suboptimal arm. As the adversary observes the learner's action before choosing the attack, this attack policy may cost \(O(\log T)\) or sublinear \(O(T^\alpha)\) budget for some \(\alpha < 1\), depending on the specific algorithm.
However, in the corruption model, as the adversary does not know which arm is pulled before corrupting the rewards, the adversary has to corrupt the optimal arm's reward realization in all \(T\) rounds to achieve the same effect, resulting in \(O(T)\) cost.
From an abstract perspective, the attack model is more powerful than the corruption model because the sequence of observing the pulled arm before attacking makes the randomized strategy (widely used in bandits) invalid, for example, randomly pulling arms in EXP3~\citep{auer2002nonstochastic} and BARBAR~\citep{gupta2019better}. Without randomization, devising robust algorithms for attacks is much more challenging than that for corruptions.

% !TeX root = ../robust-to-attack.tex
\section{Lower Bounds}\label{sec:lower-bound}

This section first presents a new general lower bound for the adversarial attack model on stochastic multi-armed bandits (Theorem~\ref{thm:KC-lower-bound}) and two of its variants. Later, we derive two lower bounds for two special classes of bandit algorithms (Propositions~\ref{prop:additive-lower-bound} and~\ref{prop:multiplicative-lower-bound}).
All proofs are deferred to Appendix~\ref{sec:proof_for_lower_bound}.

\subsection{A General Lower Bound}

We first present a general lower bound \(\Omega(KC)\) in Theorem~\ref{thm:KC-lower-bound}. Together with known lower bounds of \mab, we derive Proposition~\ref{prop:combined-lower-bound}.
\S\ref{sec:algorithm-known} presents algorithms matching these lower bounds.

\begin{restatable}{theorem}{KClowerbound}
\label{thm:KC-lower-bound}
    Given a stochastic multi-armed bandit game with \(K\) arms, under attack with budget \(C\), and \(T>KC\) decision rounds,\footnote{Note that \(T\le KC\) implies that \(C=\Omega(T)\) which trivially results in a linear regret.} for any bandits algorithm, there exists an attack policy with budget \(C\) that can make the algorithm suffer \(\Omega(KC)\) regret.
\end{restatable}

% \begin{restatable}{proposition}{combinedlowerbound}
\begin{proposition}
\label{prop:combined-lower-bound}
    For gap-dependent regret bounds and any consistent bandit algorithm\footnote{When \(C=0\), a consistent bandits algorithm has regret \(\bar R_T \le O(T^\alpha)\) for any \(\alpha \in (0,1)\).}, the lower bound is roughly \(\Omega\left( \sum_{k\neq k^*} \frac{\log T}{\Delta_k} + KC \right)\), or formally, for some universal constant \(\xi>0\),
    \begin{equation}
        \label{eq:gap-dependent-lower-bound}
        \liminf_{T\to\infty} \frac{\bar R_T - \xi KC}{\log T} \ge \sum_{k\neq k^*} \frac{1}{2\Delta_k}.
    \end{equation}
    For gap-independent cases, the lower bound is \begin{equation}
        \label{eq:gap-independent-lower-bound}
        \bar R_T \ge \Omega(\sqrt{KT} + KC).
    \end{equation}
\end{proposition}

\subsection{Two Lower Bounds for Special Algorithm Classes}

In this subsection, we first recall results from adversarial attacks on bandits literature~\citep{liu2019data,zuo2024near} in Theorem~\ref{thm:linear-regret-condition}. Then, based on these results, we derive two lower bounds for bandit algorithms with additive and multiplicative regret bounds respectively. In \S\ref{sec:algorithm-unknown}, we present algorithms that match these lower bounds.

\begin{theorem}[{Adapted from~\citep[Fact 1]{zuo2024near} and~\citep[Theorem 4.12]{he2022nearly}}]\label{thm:linear-regret-condition}
    If an algorithm achieves regret \(R_T\) when total attack \(C=0\), then there exists an attack policy with \(C=\Theta(R_T)\) that can make the algorithm suffer linear regret.
\end{theorem}

\begin{proposition}[Achievable additive regret bounds]\label{prop:additive-lower-bound}
    Given any parameter \(\alpha \in [\frac 1 2, 1)\), for any bandit algorithm that achieves an additive regret upper bound of the form \(O(T^\alpha + C^\beta)\), we have \(\beta \ge \frac{1}{\alpha}\).
\end{proposition}

\begin{proposition}[Achievable multiplicative regret bounds]\label{prop:multiplicative-lower-bound}
    Given parameter \(\alpha \in [\frac 1 2, 1)\), for any bandit algorithm that achieves a multiplicative regret bound of the form \(O(T^\alpha C^\beta)\), we have \(\beta \ge \frac{1}{\alpha} - 1\).
\end{proposition}

% \begin{itemize}
%     \item If \(C'\le \log T\), then we have \(O(K\log T)\) or \(O(T)\) both of which are optimal in corresponding scenarios.
%     \item If \(C' > \log T\), then we have \(O(KC')\) or \(O(T)\) regret, which are also optimal.
% \end{itemize}

% !TeX root = ../robust-to-attack.tex
\section{Algorithms with Known Attack Budget}\label{sec:algorithm-known}

In this section, we study the stochastic multi-armed bandit problem with a known attack budget. In \S\ref{subsec:algorithm-known-additive}, we first present an algorithm with an additive gap-dependent upper bound on the attack budget. Then, in \S\ref{subsec:algorithm-known-multiplicative}, we modify this algorithm to two variants with gap-independent upper bounds, one with an additive \(C\) term and another with a multiplicative \(C\) term.

\subsection{\sewr: An Algorithm with Gap-Dependent Upper Bound}\label{subsec:algorithm-known-additive}

Given the knowledge of attack budget \(C\), one is able to predict the worst-case attack and design an algorithm to defend against it.
Here, the robustness is achieved by widening the confidence radius of the reward estimate to account for the \(C\) attack such that the corresponding widened confidence interval contains the true reward mean with a high probability (as if the rewards were not attacked).
Denote \(\tilde{\mu}_k\) as the empirical mean of the attacked reward observations of arm \(k\). The widened confidence interval is centered at \(\tilde{\mu}_k\) and has a radius of \(\sqrt{{\log(2/\delta)}/{N_k}} + {C}/{N_k}\), where \(N_k\) is the number of times arm \(k\) is pulled and \(\delta\) is the confidence parameter.
Then, we utilize the successive elimination (\texttt{SE}) algorithm~\citep{even2006action} to devise a robust bandit algorithm in Algorithm~\ref{alg:se-wcr}.
The algorithm maintains a candidate set of arms \(\mathcal{S}\), initialized as the set of all arms \(\mathcal{K}\),
and successively eliminate suboptimal arms according to the widened confidence intervals (Line~\ref{line:elimination-sewr}), called successive elimination with wide radius (\sewr).
A similar algorithm design idea was also used in~\cite {lykouris2018stochastic} for the stochastic bandit problem with corruption attacks. Theorem~\ref{thm:addtive-upper-bound-known-attack} provides the regret upper bound of Algorithm~\ref{alg:se-wcr}.

% and then, we utilize the successive elimination (\texttt{SE}) algorithm~\citep{even2006action} to eliminate suboptimal arms.
% , and, therefore, the optimal arm \(k^*\) will not be mistakenly eliminated due to the attack.

% Specifically, we utilize the successive elimination (\texttt{SE}) algorithm~\citep{even2006action} and widen the confidence radius by \(C/N\) to account for the attack, where \(N\) is the number of times an arm is pulled.

% Here, we utilize the
\begin{algorithm}[tp]
    \caption{\sewr: Successive Elimination with Wide Confidence Radius}
    \label{alg:se-wcr}
    \begin{algorithmic}[1]
        \Input total attack \(C\) (or budget \(\CInput\)), number of arms \(K\), decision rounds \(T\), parameter \(\delta\)
        \State Initialize candidate arm set \(\mathcal{S} \gets \mathcal{K}\), number of arm pulls \(N_k\gets 0\), reward estimates \(\tilde \mu_k \gets 0\)
        \While{\(t\le T\)}\label{line:while}
        \State Uniformly pull each arm in \(\mathcal{S}\) once and observes \(\tilde X_{k}\) for all arms \(k\in\mathcal{S}\)
        \State Update parameters: \(t\gets t+\abs{\mathcal{S}}, N_k\gets N_k + 1,\tilde\mu_k\gets \frac{\tilde\mu_k (N_k - 1) + \tilde X_{k}}{N_k}\) for all arms \(k\in \mathcal{S}\)
        \State \(\tilde{\mu}_{\max} \gets \max_k \tilde{\mu}_k\)
        \State \(\mathcal{S} \gets \left\{k\in \mathcal{S}: \tilde\mu_k \ge \tilde\mu_{\max} - 2\left( \sqrt{\frac{\log(2KT/\delta)}{N_{k}}} + \frac{C}{N_k} \right)\right\}\) \label{line:elimination-sewr}
        \EndWhile
    \end{algorithmic}
\end{algorithm}

\begin{restatable}{theorem}{addtiveknownbound}
    % \begin{theorem}
    \label{thm:addtive-upper-bound-known-attack}
    Given the total attack or its upper bound \(C\), Algorithm~\ref{alg:se-wcr}'s the realized regret is upper bounded by \(
        R_T \le O\left( \sum_{k\neq k^*} \frac{\log(KT/\delta)}{\Delta_k} + KC\right) \text{ with probability } 1-\delta,
    \)
    and, with \(\delta=K/T\), its pseudo regret is upper bounded as \(
    \bar R_T \le O\left( \sum_{k\neq k^*} \frac{\log T}{\Delta_k} + KC\right).
    \)
    % \end{theorem}
\end{restatable}

The regret upper bound in Theorem~\ref{thm:addtive-upper-bound-known-attack} matches the lower bound in~\eqref{eq:gap-dependent-lower-bound} in terms of both the first classic regret term and the second additive attack \(\Omega(KC)\) term. This shows that the regret upper bound is nearly optimal up to some prefactors.
Our regret upper bounds improve the \(O\left( \sum_{k\neq k^*} \frac{\log(KT/\delta)}{\Delta_k} + \sum_{k\neq k^*} \frac{C}{\Delta_k}\right)\) in~\citet[Theorem 1]{lykouris2018stochastic} by removing the multiplicative dependence of \(1/\Delta_k\) on the total attack \(C\) term.
This improvement comes from a tighter analysis for sufficient pulling times for eliminating a suboptimal arm (see Lemma~\ref{lma:tighter-sample-complexity} in Appendix~\ref{app:proof-known}).
This improvement is crucial for achieving our tight multiplicative regret bounds for both known and unknown attack budget cases (Sections~\ref{subsec:algorithm-known-multiplicative} and~\ref{subsec:algorithm-unknown-multiplicative}).

% \begin{lemma}\label{lma:confidence-interval-under-attack}
%     If pulling an arm \(k\) for \(N\) times under attack budget \(C\), then, with a probability of at least \(1-\delta\), we have
%     \[
%         \abs*{\tilde{\mu}_k - \mu_k} \le \sqrt{\frac{\log(2/\delta)}{N}} + \frac{C}{N}.
%     \]
% \end{lemma}
% \todo{Decide whether to include Lemma~\ref{lma:confidence-interval-under-attack}.}

\subsection{\sewrst: An Algorithm with Gap-Independent Upper Bounds}
\label{subsec:algorithm-known-multiplicative}

We follow the approach of stochastic bandits literature~\citep[Chapter 6]{lattimore2020bandit} to transfer the algorithm with a gap-dependent upper bound to another algorithm with gap-independent bounds. The main idea is to replace the while loop condition in Line~\ref{line:while} in Algorithm~\ref{alg:se-wcr} with nuanced stopping conditions.
If the stopping condition is triggered before the end of the \mab game, i.e., \(t=T\), the algorithm will uniformly pull the remaining arms in the candidate set \(\mathcal{S}\) until the end of the game. As the pseudo-code of this algorithm is similar to \sewr, we defer it to Algorithm~\ref{alg:se-wcr-stop} in Appendix~\ref{app:addition-algorithm}.

We consider two kinds of stopping conditions.
The first condition aims to transfer the gap-dependent bound to an independent one while keeping the additive \(KC\) term in the upper bound.
We start from a stopping condition \(N_k \le \frac{\log (KT/\delta)}{\epsilon^2} + \frac{C}{\epsilon}\), where \(\epsilon\) is a parameter to be determined. With this condition and the proof details of Theorem~\ref{thm:addtive-upper-bound-known-attack}, we can upper bound the realized regret as follows,
\[
    \begin{split}
        R_T
        % & \le O\left( \sum_{k: \Delta_k > \epsilon} \frac{\log (KT/\delta)}{\Delta_k} + KC + \sum_{k:\Delta_k\le \epsilon} \epsilon N_k\right)
        \le O\left( \frac{K\log (KT/\delta)}{\epsilon}+ \epsilon T + KC \right)
        \le O\left( \sqrt{KT\log (KT/\delta)} + KC \right),
    \end{split}
\]
where we choose \(\epsilon = \sqrt{K\log (KT/\delta)/{T}}\) in the last inequality.
% That is, the stopping condition \(N_k \le \frac{T}{K} + C\sqrt{\frac{T}{K\log(KT/\delta)}}\) yields a regret upper bound of \(O( \sqrt{KT\log (KT/\delta)} + KC )\).

The second stopping condition, while transferring the gap-dependent to independent, also converts the additive \(KC\) term to multiplication. To derive it, we start from a stopping condition \(N_k \le \frac{\log (KT/\delta) + C}{\epsilon^2}\), where \(\epsilon\) is a parameter to be determined. With this condition and the proof details of Theorem~\ref{thm:addtive-upper-bound-known-attack}, we can upper bound the realized regret as follows,
\begin{equation*}
    \begin{split}
        R_T
        %  & \le
        % O\left(  \sum_{k: \Delta_k > \epsilon} \frac{\log (KT/\delta) + C}{\Delta_k} + \sum_{k: \Delta_k \le \epsilon}\epsilon N_k
        % \right)
        % \\
         & \le O\left(   \frac{K\log (KT/\delta) + KC}{\epsilon} + \epsilon T
        \right)
        % \\
        % & 
        \le O\left( \sqrt{KT(\log (KT/\delta) + C)} \right)
    \end{split}
\end{equation*}
where \(\epsilon = \sqrt{{(K\log (KT/\delta) + KC)} / {T}}\) in the last inequality.
% That is, the second stopping condition is \(N_k \le \frac{T}{K}\), which results in the \(O( \sqrt{KT(\log (KT/\delta) + C)} )\) bound.
We summarize the above results in the following theorem.

\begin{theorem}\label{thm:gap-independent-upper-bound-known-attack}
    Given the total attack or its upper bound \(C\), \sewrst (Algorithm~\ref{alg:se-wcr-stop}) with stopping condition \(N_k \le \frac{T}{K} + C\sqrt{\frac{T}{K\log(KT/\delta)}}\) has the realized regret upper bounded as follows, \begin{equation}
        \label{eq:additive-upper-bound-known-attack-independent}
        R_T \le O\left( \sqrt{KT\log (KT/\delta)} + KC \right) \text{ with probability } 1-\delta,
    \end{equation}
    and \sewrst with stopping condition \(N_k \le \frac{T}{K}\) has the realized regret upper bounded as follows,
    \begin{equation}
        \label{eq:multiplicative-upper-bound-known-attack}
        R_T \le O\left( \sqrt{KT(\log (KT/\delta) + C)} \right) \text{ with probability } 1-\delta.
    \end{equation}
\end{theorem}

Let \(\delta=K/T\), the regret in~\eqref{eq:additive-upper-bound-known-attack-independent} becomes \(O\left( \sqrt{KT\log T} + KC \right)\), matching with the lower bound in~\eqref{eq:gap-independent-lower-bound} up to logarithmic factors. This implies the regret upper bound is nearly optimal.
Let \(\delta=K/T\), the regret in~\eqref{eq:multiplicative-upper-bound-known-attack} becomes \(O( \sqrt{KT(\log T + C)})\). Although this bound cannot match the lower bound in~\eqref{eq:gap-independent-lower-bound}, its dependence on total attack \(C\) is square root, better than the linear dependence in~\eqref{eq:additive-upper-bound-known-attack-independent}. This better dependence will shine when devising algorithms for unknown attack budget in \S\ref{subsec:algorithm-unknown-multiplicative}.

\begin{remark}[Algorithm selection for known \(C\) case]
Although this section proposes three algorithms for known \(C\)---one \sewr and two \sewrst, the latter two \sewrst algorithms are mainly for the theoretical interest of deriving gap-independent and multiplicative \(C\) bounds. Practically, \sewr should have better empirical performance than that of \sewrst, as \sewrst may stop eliminating relatively good suboptimal arms earlier than \sewr, causing a slightly higher regret.
\end{remark}

% !TeX root = ../robust-to-attack.tex
\section{Algorithms with Unknown Attack Budget}\label{sec:algorithm-unknown}

This section presents algorithms that can defend against adversarial attacks with unknown attack budgets. We first present an algorithm with an additive upper bound in \S\ref{subsec:algorithm-unknown-additive} and then an algorithm with a multiplicative upper bound in \S\ref{subsec:algorithm-unknown-multiplicative}.

Before diving into the detailed algorithm designs, we recall that the input attack budget \(\CInput\) for \sewr and \sewrst can be interpreted as the level of robustness that the algorithms have.
For any actual attack \(C \le \CInput\), the algorithms achieve the regret upper bounds in Theorems~\ref{thm:addtive-upper-bound-known-attack} and~\ref{thm:gap-independent-upper-bound-known-attack} respectively; for attack \(C > \CInput\), the algorithms may be misled by the attacker and suffer linear regret.
This ``robust-or-not'' separation demands the knowledge of the actual attack \(C\) for \sewr and \sewrst to perform well.
In this section, we apply two types of ``smoothing'' algorithmic techniques to ``smooth'' this ``robust-or-not'' separation of \sewr and \sewrst so as to achieve robustness against unknown attacks. An overview of algorithm design is in Figure~\ref{fig:algorithm-design}.

On the other hand, robust algorithms designed for unknown corruptions, e.g.,~\citet{lykouris2018stochastic,gupta2019better}, do not apply to the unknown attack setting. Because these robust algorithms all rely on some randomized action mechanism to defend against corruption, which is invalid for the attack as the attacker can manipulate the rewards of the arms \emph{after} observing the learner's actions.

\subsection{\pewr: An Algorithm with Additive Upper Bound}\label{subsec:algorithm-unknown-additive}

The first ``smoothing'' algorithmic technique is applying the multi-phase structure to modify \sewr (Algorithm~\ref{alg:se-wcr}), yielding the phase elimination with wide confidence radius algorithm, \pewr (Algorithm~\ref{alg:phase-elimination-wcr}) for unknown attack budget.
This algorithm considers multiple phases of \sewr, each with a different assumed attack budget \(\CInput\), and the candidate arm set \(\mathcal{S}\) is updated consecutively among phases, that is, the arms eliminated in previous phases are not considered in later phases.

% \todo{consider setting the remarks' enumerate independent from theorems and lemmas.}

Specifically, \pewr separates the total decision rounds into multiple phases \(h\in\{1,2,\dots\}\),
each with a length doubled from the previous phase (Line~\ref{line:double-length}).
In phase \(h\), the elimination assumes the potential attack is upper bounded by \(\hat C_h\), which is halved in each phase (Line~\ref{line:halving-attack-budget}). This mechanism results in two kinds of phases. Denote \(h'\) as the first phase whose corresponding \(\hat C_{h'}\) is smaller than the true attack budget \(C\). Then, for phases \(h<h'\), we have \(\hat C_h \ge C\), and, therefore, the algorithm eliminates arms properly as \sewr;
for phases \(h\ge h'\), as \(\hat C_h < C\), the algorithm may falsely eliminate the optimal arm and suffer extra regret. But even though the optimal arm is eliminated, the algorithm in these phases still has the level of \(\hat C_h\) robustness, which guarantees some relatively good suboptimal arms---with a mean close to the optimal arm---remaining in the candidate set \(\mathcal{S}\) such that the total regret due to losing the optimal arm in these phases is not large (bounded by \(KC^2\) in the proof).
We present the pseudo-code of \pewr in Algorithm~\ref{alg:phase-elimination-wcr} and its regret upper bound in Theorem~\ref{thm:addtive-upper-bound-unknown-attack}.

% \todo{Consider changing the \(L_h / \abs{\mathcal C_h}\) to \(L_h\) only in the pseudo-code. This can avoid the dependence of \(L_0\) and will change the upper bound \(\sqrt{KT}\) to \(K\sqrt{T}\). }

\begin{algorithm}[tp]
    \caption{\pewr: Phase Elimination with Wide Confidence Radius}
    \label{alg:phase-elimination-wcr}
    \begin{algorithmic}[1]
        \Input number of arms \(K\), decision rounds \(T\), confidence parameter \(\delta\)
        \State Initialize candidate arm set \(\mathcal{S} \gets \mathcal{K}\), the number of arm pulls \(N_k\gets 0\), reward mean estimates \(\hat\mu_k \gets 0\), phase \(h\gets 0\), phase upper bound \(H \gets \ceil{\log_2 T} - 1\), initial pulling times \(L_0 \gets K\)
        \For{each phase \(h=0, 1,\dots\)}
        \State \(\hat C_h \gets \min \{\sqrt{T}/K, 2^{H-h-1}\}\) \label{line:halving-attack-budget}
        \State Pull each arm in \(\mathcal{S}\) for \({L_h}/{\abs{\mathcal{S}_h}}\) times
        \State Update arm empirical means \(\hat\mu_{k,h}\) for all arms \(k\in\mathcal{S}\)
        \State \(\mathcal{S}_{h+1} \gets \left\{k\in \mathcal{S}_h: \hat\mu_{k,h} \ge \max\limits_{k'\in\mathcal{S}_h} \hat\mu_{k',h} - 2\left( \sqrt{\frac{\log(2KH/\delta)}{L_h / \abs{\mathcal C_h}}} + \frac{\hat C_h}{L_h/\abs{\mathcal C_h}} \right)\right\}\) \label{line:elimination}
        \State \(L_{h+1} \gets 2L_h \) \label{line:double-length}
        \EndFor
    \end{algorithmic}
\end{algorithm}

% \todo{check, whether \(H = \ceil{\log_2 T}\) or \(\ceil{\log_2 T} - 1\). It should be the latter one!}

\begin{restatable}{theorem}{addtiveunknownbound}
    % \begin{theorem}
    \label{thm:addtive-upper-bound-unknown-attack}
    Algorithm~\ref{alg:phase-elimination-wcr} has a realized regret upper bound, with a probability of at least \(1-\delta\),
    \(
        % \begin{split}
            R_T \le {O}\left( \sqrt{KT\log\left( {K \log T}/{\delta} \right)} + \sqrt{T} \log T+  KC\log T + KC^2\right)
            % \\
            % & = \tilde{O}\left( \sqrt{KT\log (1/\delta)} + KC^2\right),
        % \end{split}
    \)
    as well as a pseudo regret upper bound as follows, with \(\delta = {K}/{ T}\),
    \[
        \begin{split}
            \bar R_T & \le {O}\left( \sqrt{KT\log\left({T \log T} \right)} + \sqrt{T} \log T+  KC\log T + KC^2\right)
            = \tilde{O}\left( \sqrt{KT} + KC^2\right).
        \end{split}
    \]
    % \end{theorem}
\end{restatable}

Comparing our upper bound in Theorem~\ref{thm:addtive-upper-bound-unknown-attack} to the lower bound in Proposition~\ref{prop:additive-lower-bound} with \(\alpha = {1}/{2}\), \(\Omega(\sqrt{T} + C^2)\), we know that our upper bound is tight in terms of both \(T\) and \(C\). Compared with the lower bound \(\Omega(\sqrt{KT} + KC)\) in Theorem~\ref{prop:combined-lower-bound}, our upper bound is also tight in terms of \(K\).

This multi-phase elimination idea is widely used in bandit literature, e.g., batched bandits~\citep{gao2019batched}, multi-player bandits~\citep{wang2019distributed}, and linear bandits~\citep{bogunovic2020corruption}. The most related to ours is the phase elimination algorithm proposed in~\citet{bogunovic2020corruption} for linear bandits. However, directly reducing their results to our stochastic bandit case would yield a \(\tilde{O}(K\sqrt{KT} + C^2)\) upper bound for \(C\le \frac{\sqrt{T}}{(K\log\log K)\log T}\) only, and, for general \(C\), their results would become \(\tilde{O}(K\sqrt{KT} + K^2C^2)\).
Our result in Theorem~\ref{thm:addtive-upper-bound-unknown-attack} is tighter than theirs by a factor of the number of arms \(K\) on the first \(\sqrt{T}\) term, and, for the general \(C\) case, our result is tighter by a factor of \(K\) on the second \(C^2\) term as well.
This improvement comes from the tighter confidence bound employed in Line~\ref{line:elimination} of Algorithm~\ref{alg:phase-elimination-wcr} and the corresponding tighter analysis for the necessary number of observations for the algorithm to eliminate arms (Lemma~\ref{lma:tighter-sample-complexity}) in \S\ref{subsec:algorithm-known-additive}.

\subsection{\mssewr: An Algorithm with Multiplicative Upper Bound}\label{subsec:algorithm-unknown-multiplicative}

% \todo{May add the algorithm pseudo-code for Model selection case in the appendix. or not}

While the additive bound in Theorem~\ref{thm:addtive-upper-bound-unknown-attack} can be transferred to a multiplicative bound \(\tilde{O}(KC^2\sqrt{T})\),
this bound cannot match the lower bound in Proposition~\ref{prop:multiplicative-lower-bound}, e.g., \(\Omega(C\sqrt{T})\) for \(\alpha = \frac 1 2\).
In this section, we deploy another algorithmic technique, model selection, to \sewrst (Algorithm~\ref{alg:se-wcr}) to achieve tight multiplicative upper bounds for the unknown attack budget case.
Unlike the multi-phase technique, where each phase assumes a different attack budget, we consider multiple algorithm instances with different attack budget inputs, each of which is a ``model.''
Model selection means selecting the model (algorithm instance) that performs best in the unknown environment. In our scenario, it is to select the instance with the smallest attack budget input \(\CInput\) that is larger than the actual attack budget \(C\).

Specially, we consider \(G\coloneqq\ceil{\log_2 T}\) instances of \sewrst, each with a different attack budget input \(\CInput=2^g\), \(g=1,2,\dots,G\).
Among these instances, the best is \(g^* = \ceil{\log_2 C}\), as \(2^{g^*} \in [C, 2C)\) is the smallest attack budget larger than the actual attack \(C\).
We then apply a model selection algorithm, e.g., \texttt{CORRAL}~\citep{agarwal2017corralling} or \texttt{EXP3.P}~\citep{auer2002nonstochastic}, to select the best one among the \(G\) instances.
Algorithm~\ref{alg:model-selection} presents the pseudo-code.
To prove the regret upper bound, we recall a simplified version of~\citet[Theorem 5.3]{pacchiano2020model} in the following lemma.

\begin{algorithm}[tp]
    \caption{\mssewr: Model Selection with Wide Confidence Radius}
    \label{alg:model-selection}
    \begin{algorithmic}[1]
        \Input number of arms \(K\), decision rounds \(T\), model selection algorithm \(\texttt{ModSelAlg}\)
        % e.g., \texttt{CORRAL}~\citep{agarwal2017corralling} or \texttt{EXP3.P}~\citep{auer2002nonstochastic}
        \State Construct \(G\coloneqq \ceil{\log_2 T}\) base algorithm instances as
        \newline \null\hspace{3cm} \(
            \mathcal{B} \gets \{\sewrst(C=2^g, K, T, \delta= K/ T): g=1,2,\dots, G\}.
        \)
        \State \(\texttt{ModSelAlg}(\mathcal{B})\) for \(T\) decision rounds
    \end{algorithmic}
\end{algorithm}

\begin{lemma}\label{lma:model-selection}
    If base algorithm instance has regret upper bound \(\bar R_T \le T^\alpha c(\delta)\) with probability \(1-\delta\) for known constant \(\alpha \in [\frac 1 2, 1)\) and the unknown function \(c:\mathbb{R}\to\mathbb{R}\), 
    % (e.g., due to unknown model parameters), 
    then the regrets of \emph{\texttt{CORRAL}} and \emph{\texttt{EXP3.P}} are \(\bar R_T \le \tilde{O}(c(\delta)^{\frac{1}{\alpha}} T^{\alpha})\) and \(\bar R_T \le \tilde{O}(c(\delta)^{\frac{1}{2-\alpha}} T^{\alpha})\) respectively.
\end{lemma}

To apply this lemma to our case, we first convert the gap-independent upper bound in~\eqref{eq:multiplicative-upper-bound-known-attack} of \sewrst (Algorithm~\ref{alg:se-wcr-stop}) in Theorem~\ref{thm:gap-independent-upper-bound-known-attack} to a multiplicative upper bound: When the stop condition is \(N_k \le  \frac{T}{K}\), the realized regret upper bound is \(\bar R_T \le {O}(\sqrt{KTC\log (KT/\delta)})\) with a probability of at least \(1-\delta\). Plugging this into Lemma~\ref{lma:model-selection} with \(\alpha = \frac{1}{2}\) and \(c(\delta) = \sqrt{C\log(KT/\delta)}\) and letting \(\delta = K/T\), we have the pseudo regret upper bounds for \mssewr as follows:

\begin{theorem}\label{thm:multiplicative-upper-bound-unknown-attack}
    Deploying \sewrst (stop condition \(N_k \le T/K\)) as the base algorithm instance,
    Algorithm~\ref{alg:model-selection} with \emph{\texttt{CORRAL}} has pseudo regret upper bound \(\bar R_T \le \tilde{O}(\sqrt{KC}T^{\frac 2 3})\),
    and Algorithm~\ref{alg:model-selection} with \emph{\texttt{EXP3.P}} has pseudo regret upper bound \(\bar R_T \le \tilde{O}(KC\sqrt T)\).
\end{theorem}

The bounds in Theorem~\ref{thm:multiplicative-upper-bound-unknown-attack} are tight in terms of the trade-off of \(T\) and \(C\), as they respectively match the lower bounds in Proposition~\ref{prop:multiplicative-lower-bound}, \(\Omega(\sqrt{C} T^{\frac 2 3})\) when \(\alpha = \frac 2 3\) and \(\Omega(C\sqrt{T})\) when \(\alpha = \frac 1 2\).
The challenge of achieving these tight regret bounds in Theorem~\ref{thm:multiplicative-upper-bound-unknown-attack} is in discovering the right base algorithm instances with suitable multiplicative upper bounds.
For example, the first upper bound in Theorem~\ref{thm:gap-independent-upper-bound-known-attack} can also be converted to the multiplicative 
\(\bar R_T \le {O}(C\sqrt{KT\log (KT/\delta)})\) bound with a probability of at least \(1-\delta\).
However, applying the model selection technique to this bound only results in \(\tilde{O}(C\sqrt{K} T^{\frac 2 3})\) and \(\tilde{O}(C^2 K\sqrt T)\) bounds, which are not as tight as the ones in Theorem~\ref{thm:multiplicative-upper-bound-unknown-attack}.

This model selection technique was also used in~\citet{kang2024robust} to boost the algorithm for Lipchitz bandits with a known attack budget to an algorithm with an unknown attack budget. However, their regret upper bounds, mapping to our \mab setting, are \(\tilde{O}(C^{\frac{1}{K+2}}T^{\frac{K+2}{K+3}})\) and \(\tilde{O}(C^{\frac{1}{K+1}}T^{\frac{K+1}{K+2}})\), which have a much worse dependence on \(T\) than the ones in Theorem~\ref{thm:multiplicative-upper-bound-unknown-attack}, especially when \(K\) is large.
This is because the regret upper bounds of their base algorithms have a much worse dependence on \(T\) than our base \sewrst algorithm.
This highlights the importance of devising the suitable base algorithm instances for \mab with the known attack budget in \S\ref{subsec:algorithm-known-multiplicative}.
% possible due to that their known attack budget algorithm does not have a tight multiplicative regret upper bound analysis, is not as tight as the one we present in this paper. This observation also suggests that looking back to case of known attack budget and devising algorithms with tighter bounds may be the right direction for finally achieving tighter regret bounds for structured bandits under unknown adversarial attacks. We leave this as an open problem for future research.

\begin{remark}[Regret comparison for unknown \(C\) case in Figure~\ref{fig:unknown-C-comparison}]\label{remark:regret-bounds-comparison}
    The section presents three different upper bounds for unknown attacks, one additive \(\tilde{O}(\sqrt{KT} + KC^2)\) by \pewr and two multiplicative \(\tilde{O}(C\sqrt{K}T^{\frac 2 3}), \tilde{O}(C^2 K \sqrt T)\) by two types of \mssewr. 
    Giving attack \(C = T^\beta\) for some parameter \(\beta\in [0,1]\), these three bounds can be written in the form of \(\tilde{O}(T^\gamma)\) for corresponding exponents \(\gamma\).
    Figure~\ref{fig:unknown-C-comparison} compares these bounds in terms of the exponent \(\gamma\) by varying \(\beta\).
    Comparing all three bounds, the additive bound is better than both multiplicative ones when \(\beta \le \frac 4 9\), and the \(\tilde{O}(C\sqrt{K}T^{\frac 2 3})\) bound is better than the additive one when \(\beta > \frac 4 9\).
    Comparing among the two multiplicative bounds, the \(\tilde{O}(C^2 K \sqrt T)\) is better than \(\tilde{O}(C\sqrt{K}T^{\frac 2 3})\) when \(\beta \le \frac 1 3\), and vice versa.
    Since Figure~\ref{fig:unknown-C-comparison} suggests that no single algorithm consistently outperforms the others, and without prior knowledge of \(C\), there is no definitive rule for selecting a specific algorithm in practice.
\end{remark}

% \begin{remark}
%     Although both the multiplicative bounds in Theorem~\ref{thm:multiplicative-upper-bound-unknown-attack} and the additive bounds in Theorem~\ref{thm:addtive-upper-bound-unknown-attack} are tight in terms of the trade-off of \(T\) and \(C\). Their tightness is with respect to special cases of algorithms as stated in Propositions~\ref{prop:additive-lower-bound} and~\ref{prop:multiplicative-lower-bound}.
%     Other algorithms may achieve tight regret bounds in other special cases. We leave this as an open problem for future research.
% \end{remark}
% \input{sections/experiment.tex}
% !TeX root = ../robust-to-attack.tex
% \section{Conclusion}

% \textbf{Conclusion remark.}
% This paper makes a comprehensive study of bandit algorithms robust to adversarial attacks: two cases of attack (with or without the knowledge of the attack budget) and two types of regret bounds (additive or multiplicative dependence on total attack \(C\)). For each of these four scenarios, we propose corresponding algorithms with their regret upper bounds, and the regret upper bounds are then shown to be tight by providing regret lower bounds.

\textbf{Conclusion Remark}
This paper presents a comprehensive study of bandit algorithms robust to adversarial attacks, examining two scenarios: attacks with or without knowledge of the attack budget, and two types of regret bounds—additive and multiplicative dependence on the total attack \(C\). For each of these four cases, we propose corresponding algorithms and establish their regret upper bounds, which are subsequently demonstrated to be tight by providing regret lower bounds.

% \begin{ack}
%     Acknowledgement (hide by the package)
% \end{ack}

\bibliography{reference}
\bibliographystyle{plainnat}

%%%%%%%%%%%%%%%%%%%%%%%%%%%%%%%%%%%%%%%%%%%%%%%%%%%%%%%%%%%%

\newpage
\appendix

\section{Other Model Details: Corruption Process, Regret Discussion}\label{apx:corruption}

In this section, we provide more details on the corruption process. At each round \(t\in\{1,2,\dots, T\}\), the adversary observes the realized rewards \(X_{k,t}\) for all arms \(k\), as well as the rewards and actions of the learner in previous rounds. The adversary then chooses the corrupted rewards \(\tilde{X}_{k,t}\) for all arms \(k\). The learner pulls an arm \(I_t\) (maybe randomly) and observes the attacked reward \(\tilde{X}_{I_t, t}\). We summarize the decision process in Procedure~\ref{proc:decision-under-corruption}.
We denote the total amount of corruption as \[
    C' \coloneqq \sum_{t=1}^T \max_k \abs{X_{k,t} - \tilde{X}_{k,t}},
\]
where \(I_t\) is the pulled arm at time slot \(t\).

\floatname{algorithm}{Procedure}
\begin{algorithm}
    \caption{Decision under Corruptions}\label{proc:decision-under-corruption}
    \begin{algorithmic}[1]
        \For{each round \(t=1,2,\dots,T\)}
        \State Stochastic rewards $X_{k,t}$ are drawn from all arms $k$
        \State The adversary observes the realized rewards $X_{k,t}$ \newline\null\hfill
        as well as the rewards and actions of the learner in previous rounds
        \State The adversary chooses the corruption rewards $\tilde{X}_{k,t}$ for all arms $k$
        \State The learner pulls an arm $I_t$ (maybe randomly) and observes the attacked reward $\tilde{X}_{I_t, t}$
        \EndFor
    \end{algorithmic}
\end{algorithm}
\floatname{algorithm}{Algorithm}

\begin{remark}[Regret defined by corrupted rewards]
    Beside the regret defined on raw reward in~\eqref{eq:realized-regret}, one can also define the regret based on the corrupted rewards as follows, \(
    \tilde{R}_T \coloneqq \max_{k\in\mathcal{K}} \sum_{t\in\mathcal{T}} (\tilde{X}_{k,t} - \tilde{X}_{I_t, t}).
    \)
    This definition is considered in~\citet{zimmert2021tsallis,bogunovic2020corruption}. As the adversary can manipulate at most \(C\) rewards, one can easily show that, \(
    \abs{\tilde R_T - R_T} \le 2C.
    \)
    We will focus on the regret defined by raw rewards~\eqref{eq:realized-regret} in the following, as most results in this paper have at least a linear dependence on \(C\).
\end{remark}

\section{Proof for Lower Bound} \label{sec:proof_for_lower_bound}

\KClowerbound*

\begin{proof}[Proof of Theorem~\ref{thm:KC-lower-bound}]
    We consider \(K\) instances of a \mab game where for instance \(\mathcal{I}_k\) for any \(k\in\{1,2,\dots,K\}\), the arm \(k\) has reward \(\epsilon > 0\) and all other arms have reward \(0\) without any noise.

    The adversary's policy is that whenever the algorithm pulls an arm with a non-zero reward, the adversary attacks the algorithm by changing the reward of the arm to \(0\). The adversary can conceal the optimal arm for \(\floor{\frac{C}{\epsilon}}\) pulls.

    After \(\floor{\frac C \epsilon} \floor{\frac K 2}\) time slots, there exists at least \(\floor{\frac K 2}\) arms that are pulled at most \(\floor{\frac C \epsilon}\) times. That is, if the optimal arm is among these \(\floor{\frac K 2}\) arms, the algorithm cannot identify the optimal arm.
    Let us pick the instance whose optimal arm is among these \(\floor{\frac K 2}\) arms. Then, during these \(\floor{\frac C \epsilon} \floor{\frac K 2}\) time slots, the algorithm needs to pay the regret of \[
        \floor*{\frac C \epsilon} \left(  \floor*{\frac K 2} - 1  \right) \epsilon = \Omega(KC).
    \]
\end{proof}

\textbf{Proposition~\ref{prop:combined-lower-bound}}
% \label{prop:combined-lower-bound}
    \textit{ For gap-dependent regret bounds and any consistent bandit algorithm\footnote{When \(C=0\), a consistent bandits algorithm has regret \(\bar R_T \le O(T^\alpha)\) for any \(\alpha \in (0,1)\).}, the lower bound is roughly \(\Omega\left( \sum_{k\neq k^*} \frac{\log T}{\Delta_k} + KC \right)\), or formally, for some universal constant \(\xi>0\),
    \[
        \liminf_{T\to\infty} \frac{\bar R_T - \xi KC}{\log T} \ge \sum_{k\neq k^*} \frac{1}{2\Delta_k}.
    \]
    For gap-independent cases, the lower bound is \[
        \bar R_T \ge \Omega(\sqrt{KT} + KC).
    \]}

\begin{proof}[Proof of Proposition~\ref{prop:combined-lower-bound}]
    From bandits literature, e.g.,~\citet[Chapter 16]{lattimore2020bandit}, we know that the pseudo regret of any consistent bandit algorithm can be lower bounded as follows, \[\liminf_{T\to\infty} \frac{\bar R_T}{\log T} \ge \sum_{k\neq k^*} \frac{1}{\Delta_k}.\]
    On the other hand, we can rewrite the \(\Omega(KC)\) lower bound in Theorem~\ref{thm:KC-lower-bound} as follows, \[
        \liminf_{T\to\infty} \frac{\bar R_T - \xi KC}{\log T} \ge 0,
    \]
    where \(\xi>0\) is a universal constant. Adding the above two inequalities and dividing both side by \(2\), we have \[
        \liminf_{T\to\infty} \frac{\bar R_T - (\xi/2) KC}{\log T} \ge \sum_{k\neq k^*} \frac{1}{2\Delta_k}.
    \]

    For the gap-independent case, recall in bandits literature~\citep[Chapter 15]{lattimore2020bandit}, we have the \(\bar R_T \ge \Omega(\sqrt{KT})\) lower bound.
    Combining it with the \(\Omega(KC)\) lower bound in Theorem~\ref{thm:KC-lower-bound} yields \(\bar R_T \ge \Omega(\max\{\sqrt{KT}, KC\}) = \Omega(\sqrt{KT} + KC)\) lower bound.
\end{proof}

\textbf{Proposition~\ref{prop:additive-lower-bound}} (Achievable additive regret bounds)
\textit{Given any parameter \(\alpha \in [\frac 1 2, 1)\), for any bandit algorithm that achieves an additive regret upper bound of the form \(O(T^\alpha + C^\beta)\), we have \(\beta \ge \frac{1}{\alpha}\).}

\textbf{Proposition~\ref{prop:multiplicative-lower-bound}} (Achievable multiplicative regret bounds)
% \begin{proposition}[Achievable multiplicative regret bounds]\label{prop:multiplicative-lower-bound}
\textit{Given parameter \(\alpha \in [\frac 1 2, 1)\), for any bandit algorithm that achieves a multiplicative regret bound of the form \(O(T^\alpha C^\beta)\), we have \(\beta \ge \frac{1}{\alpha} - 1\).}
% \end{proposition}

\begin{proof}[Proof of Propositions~\ref{prop:additive-lower-bound} and~\ref{prop:multiplicative-lower-bound}]
    We use contradiction to prove both propositions.
    We first prove Proposition~\ref{prop:additive-lower-bound}. Given Theorem~\ref{thm:linear-regret-condition}, we know that there exists an attack policy with \(C=\Theta(R_T)\) that can make the algorithm suffer linear regret. Thus, if the algorithm achieves an additive regret upper bound of the form \(O(T^\alpha + C^\beta)\) with \(\beta<\frac{1}{\alpha}\), then the algorithm only suffers a sublinear regret when \(C=T^\alpha\), which contradicts Theorem~\ref{thm:linear-regret-condition}. Thus, the parameter \(\beta \ge \frac{1}{\alpha}\).
    Proposition~\ref{prop:multiplicative-lower-bound} can be proved via a similar contradiction argument.
\end{proof}

\section{Deferred Algorithm Pseudo-Code}\label{app:addition-algorithm}

\begin{algorithm}
    \caption{\sewrst: Successive Elimination with Wide Confidence Radius and Stop Condition}
    \label{alg:se-wcr-stop}
    \begin{algorithmic}[1]
        \Input total attack \(C\), number of arms \(K\), decision rounds \(T\), confidence parameter \(\delta\)
        \State Initialize candidate arm set \(\mathcal{S} \gets \mathcal{K}\), the number of arm pulls \(N_k\gets 0\), and reward mean estimates \(\hat\mu_k \gets 0\)
        \While{\(N_k \le \frac{T}{K}\) and \(t\le T\)}\label{line:stop-condition} \Comment{or \(N_k \le \frac T K + C\sqrt{\frac{T}{K\log (KT/\delta)}}\)}
        \State Uniformly pull each arm in \(\mathcal{S}\) once and observes \(X_k\) for all arms \(k\in\mathcal{S}\)
        \State Update parameters: \(t\gets t+\abs{\mathcal{S}}, N_k\gets N_k + 1,\hat\mu_k\gets \frac{\hat\mu_k (N_k - 1) + X_{k}}{N_k}\)
        \State \(\hat{\mu}_{\max} \gets \max_k \hat{\mu}_k\)
        \State \(\mathcal{S} \gets \left\{k\in \mathcal{S}: \hat\mu_k \ge \hat\mu_{\max} - 2\left( \sqrt{\frac{\log(2KT/\delta)}{N_{k}}} + \frac{C}{N_k} \right)\right\}\)
        \EndWhile
        \State Uniformly pull arms in \(\mathcal{S}\) till the end of the game
    \end{algorithmic}
\end{algorithm}

\section{Proof for Upper Bounds with Known Attack}\label{app:proof-known}

\addtiveknownbound*

% \todo{recall the main theorem statement in the appendix}

\begin{proof}[Proof of Theorem~\ref{thm:addtive-upper-bound-known-attack}]
    We first recall the following lemma from~\citet{lykouris2018stochastic}.
    \begin{lemma}[{\citet[Lemma 3.1]{lykouris2018stochastic}}]\label{lma:optimal-arm-not-eliminiated}
        With a probability of at least \(1-\delta\), the optimal arm \(k^*\) is never eliminated.
    \end{lemma}

    Next, we prove a lemma shows the sufficient number of pulling that are necessary for a suboptimal arm to be eliminated. This lemma is tighter than the one in~\citet[Lemma 3.2]{lykouris2018stochastic} (proved for \(N_k \ge  \frac{36\log (2KT/\delta) + 36C}{\Delta_k^2}\)), and, based on it, we derive tighter regret bounds.

    \begin{lemma}\label{lma:tighter-sample-complexity}
        With a probability of at least \(1-\delta\), all suboptimal arms \(k\neq k^*\) are eliminated after pulling each arm \(N_k \ge \frac{36\log (2KT/\delta)}{\Delta_k^2} + \frac{6C}{\Delta_k}\) times.
    \end{lemma}

    \begin{proof}[Proof for Lemma~\ref{lma:tighter-sample-complexity}]
        With Lemma~\ref{lma:optimal-arm-not-eliminiated}, we know that the optimal arm \(k^*\) is never eliminated. Thus, in the following, we show that  the suboptimal arm \(k\) will be eliminated by the optimal arm \(k^*\) on or before \(N_k \ge \frac{64\log (2KT/\delta)}{\Delta_k^2} + \frac{8C}{\Delta_k}\).

        Denote \(\hat{\mu}_k\) as the empirical mean of the raw stochastic reward observation of arm \(k\).
        By Hoeffding's inequality, we have, with a probability of at least \(1-\frac{\delta}{KT}\),
        \begin{align}\label{eq:hoeffding-inequality-application}
            \abs{\hat{\mu}_k - \mu_k}
             & \le \sqrt{\frac{\log(2KT/\delta)}{N_k}}
            \\
             & \le \frac{\Delta_k}{8},\nonumber
        \end{align}
        where the last inequality is due to \(N_k \ge \frac{64\log (2KT/\delta)}{\Delta_k^2}\).

        Comparing the true empirical mean \(\hat{\mu}_k\) with the attacked empirical mean \(\mu_k\), we have
        \[\abs{\tilde{\mu}_k - \hat{\mu}_k} \le \frac{C}{N_k} \le \frac{\Delta_k}{8},\]
        where the last inequality is due to \(N_k \ge \frac{8C}{\Delta_k}\).
        Combining the two inequalities, we have \[
            \abs{\tilde{\mu}_k - {\mu}_k} \le \frac{\Delta_k}{4}.
        \]

        Last, we show that the elimination condition for arm \(k\) would have been trigger before \(N_k \ge \frac{64\log (2KT/\delta)}{\Delta_k^2} + \frac{8C}{\Delta_k}\) as follows, \begin{align*}
            \tilde{\mu}_{k^*} - 2\left( \sqrt{\frac{\log(2KT/\delta)}{N_{k}}} + \frac{C}{N_k} \right)
             & \ge \tilde{\mu}_{k^*} - 2\left( \frac{\Delta_k}{8} + \frac{\Delta_k}{8}\right)
            \\
             & \ge \mu_{k^*} - \frac{\Delta_k}{4}          - \frac{2\Delta}{4}
            \\
             & =\mu_{k} + \frac{\Delta}{4}
            \\
             & \ge \tilde{\mu}_k.
        \end{align*}
    \end{proof}

    % The proof is similar to the proof of Theorem 1 in~\citep{lykouris2018stochastic}. We provide a sketch of the proof here for completeness.

    % With Lemma~\ref{lma:confidence-interval-under-attack}, we know that if \(N \ge \)

    % Let \(N_k\) be the number of times arm \(k\) is pulled. We have
    % \[
    %     \abs*{\hat{\mu}_k - \mu_k} \le \sqrt{\frac{\log(2/\delta)}{N_k}} + \frac{C}{N_k}
    % \]
    % with probability at least \(1-\delta\). 

    As the pseudo regret is upper bounded at most \(\sum_k \Delta_k N_k\), we have
    \begin{equation}\label{eq:regret-decomposition-upper-bound-known-attack}
        \begin{split}
            \bar R_T
             & \le \sum_{k\neq k^*} \Delta_k N_k
            \\
             & \le \sum_{k\neq k^*} \Delta_k \left( \frac{64\log(2KT/\delta)}{\Delta_k^2} + \frac{8C}{\Delta_k} \right)
            \\
             & = O\left( \sum_{k\neq k^*}  \frac{\log(KT/\delta)}{\Delta_k} + KC  \right)
            \\
             & \le O\left( \sum_{k\neq k^*} \frac{\log T}{\Delta_k} + KC\right),
        \end{split}
    \end{equation}
    where the last inequality is by choosing \(\delta = \frac{K}{T}\).

    \textbf{Transfer psudoe-regret to realized regret.}
    Next, we show the high probability bound for the realized regret \(R_T\). Multiplying \(N_k\) to both side of~\eqref{eq:hoeffding-inequality-application}, we have that, for any arm \(k\), the difference between the actual accumulated reward and its expected counterpart is upper bounded by \(\sqrt{N_k \log (2KT/\delta)}\) with a probability of at least \(1-\delta/KT\). The probability of any of above event does not hold is at most \(1-\delta\).

    For any suboptimal arm \(k\), we bound the difference as follows, \[
        \sqrt{N_k \log (2KT/\delta)} \le N_k \sqrt{\frac{\log (2KT/\delta)}{N_k}}
        \le N_k \Delta_k \sqrt{\frac{\log (2KT/\delta)}{64\log(2KT/\delta) + 8C \Delta_k}}
        \le \frac{N_k\Delta_k}{8}.
    \]

    In case that the arm \(k'\) with highest empirical mean is not the optimal arm \(k^*\), for example, this arm \(k'\) has a smaller reward gap \(\Delta_k \le O(\sqrt{{1}/{T}})\), the regret reduction due to this event is at most \(N_{k'}\Delta_{k'}\).

    Putting the potential impact of different kinds of reward realization above, we can bound the realized regret by \(O(\sum_{k\neq k^*} N_k\Delta_k)\) still. Thus, with similar derivation as~\eqref{eq:regret-decomposition-upper-bound-known-attack}, we have the realized regret of the algorithm is at most
    \[
        O\left( \sum_{k\neq k^*} \frac{\log(KT/\delta)}{\Delta_k} + KC\right) \text{ with probability } 1-\delta.
    \]
\end{proof}

\section{Proof for Upper Bounds with unknown Attack}

\addtiveunknownbound*

\begin{proof}[Proof of Theorem~\ref{thm:addtive-upper-bound-unknown-attack}]
    For \(C>\sqrt{T}/K\), the \(\tilde{O}(\sqrt{KT} + KC^2) = O(T)\) upper bound holds trivially. Hence, we only consider \(C\le \sqrt{T}/K\) in this proof.

    Note that the elimination condition in Algorithm~\ref{alg:phase-elimination-wcr} fails with a probability of at most \(\delta/KH\).
    Applying a union bound to sum all potential failure probabilities show that, with a probability of at least \(1-\delta\), the elimination condition works properly for all arms in all phases.
    % \todo{since the algorithm only has \(O(\log T)\) phases, can we replce the \(2KT/\delta\) with \(2K\log T/\delta\)?}
    In the rest of the proof, we exclude this failure probability and assume that the elimination only fails when \(C>\hat C_h\).

    Denote \(H'\) as the actually number of phases, while recall \(H = \ceil{\log_2 T}-1\) is the upper bound of the number of phases.
    % \todo{Replace \(H\) with \(H'\) in the following proof where necessary.}
    As the analysis of Algorithm~\ref{alg:se-wcr} shows, for phase \(h\) with \(\hat C_h > C\), the optimal arm \(k^*\) is in the candidate set \(\mathcal{S}_h\) with high probability.
    Denote \(h'\) as the first phase with \(\hat C_{h'} \le  C\), implying \(H - h'\le \log_2 C\).

    Denote \(k_h^*\) as the arm with highest reward mean in phase \(h\).
    Then, we have \(k_h^* = k^*\) for \(h<h'\).
    Next, we use induction to show, for \(h\ge h'\), the remaining best arm \(k_h^*\) is close to the optimal arm \(k^*\) as follows, \[
        \mu_{k^*} - \mu_{k_h^*} \le \frac{2C}{L_h/\abs{\mathcal{S}_h}}.
    \]

    If the optimal arm is eliminated in phase \(h'\), then we have \[
        \hat{\mu}_{k^*,h'} \le \max_{k\in\mathcal{S}_{h'}} \hat{\mu}_{k,h'} - 2\left( \sqrt{\frac{\log(2KH/\delta)}{L_{h'}/\abs{\mathcal{S}_{h'}}}} + \frac{\hat C_{h'}}{L_{h'}/\abs{\mathcal{S}_{h'}}} \right).
    \]

    Denote \(\hat k_{h'}^* =\argmax_{k\in\mathcal{S}_{h'}} \hat{\mu}_{k,h'}\). We have
    \begin{align*}
        \mu_{\hat{k}_{h'}^*}
         & \overset{(a)}\ge \hat{\mu}_{\hat{k}_{h'}^*,h'} - \left( \sqrt{\frac{\log(2KH/\delta)}{L_{h'}/\abs{\mathcal{S}_{h'}}}} + \frac{C}{L_{h'}/\abs{\mathcal{S}_{h'}}} \right)
        \\
         & \overset{(b)}\ge \hat{\mu}_{k^*,h'} + 2\left( \sqrt{\frac{\log(2KH/\delta)}{L_{h'}/\abs{\mathcal{S}_{h'}}}} + \frac{\hat C_{h'}}{L_{h'}/\abs{\mathcal{S}_{h'}}} \right) - \left( \sqrt{\frac{\log(2KH/\delta)}{L_{h'}/\abs{\mathcal{S}_{h'}}}} + \frac{C}{L_{h'}/\abs{\mathcal{S}_{h'}}} \right)
        \\
         & \overset{(c)}\ge \mu_{k^*} + 2\left( \sqrt{\frac{\log(2KH/\delta)}{L_{h'}/\abs{\mathcal{S}_{h'}}}} + \frac{\hat C_{h'}}{L_{h'}/\abs{\mathcal{S}_{h'}}} \right) - 2\left( \sqrt{\frac{\log(2KH/\delta)}{L_{h'}/\abs{\mathcal{S}_{h'}}}} + \frac{C}{L_{h'}/\abs{\mathcal{S}_{h'}}} \right)
        \\
         & = \mu_{k^*} - \frac{2(C - \hat C_{h'})}{L_{h'}/\abs{\mathcal{S}_{h'}}}
        \\
         & \ge \mu_{k^*} - \frac{2C}{L_{h'}/\abs{\mathcal{S}_{h'}}}
        \\
         & \ge \mu_{k^*} - \frac{2CK}{L_{h'}}
    \end{align*}
    where inequalities (a) and (c) come from the Hoeffding's inequality and the attack \(C\),
    inequality (b) is due to the elimination condition in round \(h'\).
    % \begin{itemize}
    %     \item[(a)] follows from the definition of \(\hat k_{h'}^*\),
    %     \item[(b)] follows from the definition of \(\hat k_{h'}^*\) and the elimination condition,
    %     \item[(c)] follows from the definition of \(k^*\) and the elimination condition.
    % \end{itemize}
    That is, \[
        \mu_{k^*} - \mu_{\hat{k}_{h'}^*} \le \frac{2CK}{L_{h'}}.
    \]

    With similar analysis, we can show, \(
    \mu_{\hat{k}_{h'+i}^*} - \mu_{\hat{k}_{h'+i+1}^*} \le \frac{2CK}{L_{h'+i}}, \text{ for all } i=1,\dots, H-h'-1.
    \)
    Based on the inequality, we further have, for any \(h\ge h',\)
    \begin{equation}
        \label{eq:remaining-best-arm-close-to-optimal-arm}
        \mu_{k^*} - \mu_{\hat{k}_{h}^*} \le \sum_{i=0}^{h-h'} \frac{2CK}{L_{h'+i}} =\sum_{i=0}^{h-h'} \frac{2CK}{L_0 2^{h'+i}} \le \frac{2CK}{L_0 2^{h'+1}}.
    \end{equation}

    Next, we decompose the pseudo regret as follows,
    \begin{align*}
        \bar R_T
         & = \sum_{t=1}^T (\mu_{k^*} - \mu_{I_t}) = \sum_{k\in\mathcal{K}} N_{k,T} (\mu_{k^*} - \mu_k)
        = \sum_{h=0}^{H'} \sum_{k\in\mathcal{S}_h} \frac{L_h}{\abs{\mathcal{S}_h}}(\mu_{k^*} - \mu_k)
        \\
         & \le L_0 + \sum_{h=1}^{h'} \sum_{k\in\mathcal{S}_h} \frac{L_h}{\abs{\mathcal{S}_h}}(\mu_{k^*} - \mu_k) + \sum_{h=h'+1}^{H'} \sum_{k\in\mathcal{S}_h} \frac{L_h}{\abs{\mathcal{S}_h}}(\mu_{k^*} - \mu_k)
        \\
         & = L_0 + \underbrace{\sum_{h=1}^{h'} \sum_{k\in\mathcal{S}_h} \frac{L_h}{\abs{\mathcal{S}_h}}(\mu_{k^*} - \mu_k) +
            \sum_{h=h'+1}^{H'} \sum_{k\in\mathcal{S}_h} \frac{L_h}{\abs{\mathcal{S}_h}} (\mu_{k_h^*} -  \mu_k)}_{\text{Part I}}
        \\
         & \qquad + \underbrace{\sum_{h=h'+1}^{H'} \sum_{k\in\mathcal{S}_h} \frac{L_h}{\abs{\mathcal{S}_h}} (\mu_{k^*} - \mu_{k_h^*})}_{\text{Part II}}.
    \end{align*}

    To bound Part I, we notice that all arms in \(\mathcal{S}_h\) are not eliminated in the end of the previous phase, implying, for \(h<h'\),
    \[
        \begin{split}
            \mu_k
             & \overset{(a)}\ge \hat{\mu}_k - \left( \sqrt{\frac{\log(2KH/\delta)}{L_{h-1}/\abs{\mathcal{S}_{h-1}}}} + \frac{C}{L_{h-1}/\abs{\mathcal{S}_{h-1}}} \right)
            \\
             & \overset{(b)}\ge \max_{k'\in\mathcal{S}_{h-1}} \hat \mu_{k', h-1} - 2\left( \sqrt{\frac{\log(2KH/\delta)}{L_{h-1}/\abs{\mathcal{S}_{h-1}}}} + \frac{\hat C_{h-1}}{L_{h-1}/\abs{\mathcal{S}_{h-1}}} \right)
            \\
             & \qquad - \left( \sqrt{\frac{\log(2KH/\delta)}{L_{h-1}/\abs{\mathcal{S}_{h-1}}}} + \frac{C}{L_{h-1}/\abs{\mathcal{S}_{h-1}}} \right)
            \\
             & \overset{(c)}\ge \hat \mu_{k^*, h-1} - \left( 3 \sqrt{\frac{\log(2KH/\delta)}{L_{h-1}/\abs{\mathcal{S}_{h-1}}}} + \frac{2 \hat C_{h-1} + C}{L_{h-1}/\abs{\mathcal{S}_{h-1}}} \right)
            \\
             & \overset{(d)}\ge \mu_{k^*} - \left( 4 \sqrt{\frac{\log(2KH/\delta)}{L_{h-1}/\abs{\mathcal{S}_{h-1}}}} + \frac{2 \hat C_{h-1} + 2 C}{L_{h-1}/\abs{\mathcal{S}_{h-1}}} \right),
        \end{split}
    \]
    where inequalities (a) and (d) are due to the Hoeffding's inequality and the attack \(C\), inequality (b) is due to the elimination condition, and inequality (c) is due to \(k^*\in\mathcal{S}_{h-1}\).
    That is, for arms \(k\in\mathcal{S}_h\), we have
    \begin{equation}\label{eq:bound-Delta-k}
        \mu_{k^*} - \mu_k \le 4 \sqrt{\frac{\log(2KH/\delta)}{L_{h-1}/\abs{\mathcal{S}_{h-1}}}} + \frac{2 (\hat C_{h-1} + C)}{L_{h-1}/\abs{\mathcal{S}_{h-1}}}.
    \end{equation}

    As we have \(k_h^*\in \mathcal{S}_h\), and with the same derivation as~\eqref{eq:bound-Delta-k}, we also have \begin{equation}
        \label{eq:bound-Delta-k-h-star}
        \mu_{k_h^*} -  \mu_k \le 4 \sqrt{\frac{\log(2KH/\delta)}{L_{h-1}/\abs{\mathcal{S}_{h-1}}}} + \frac{2 (\hat C_{h-1} + C)}{L_{h-1}/\abs{\mathcal{S}_{h-1}}}.
    \end{equation}

    Therefore, we have
    \begin{equation}\label{eq:bound-regret-part-I}
        \begin{split}
             & \quad \sum_{h=1}^{h'} \sum_{k\in\mathcal{S}_h} \frac{L_h}{\abs{\mathcal{S}_h}}\Delta_k
            + \sum_{h=h'+1}^{H'} \sum_{k\in\mathcal{S}_h} \frac{L_h}{\abs{\mathcal{S}_h}} (\mu_{k_h^*} -  \mu_k)
            \\
             & \overset{(a)}\le \sum_{h=1}^{H'} L_h \cdot 4 \sqrt{\frac{\log(2KH/\delta)}{L_{h-1}/\abs{\mathcal{S}_{h-1}}}} + \frac{2 (\hat C_{h-1} + C)}{L_{h-1}/\abs{\mathcal{S}_{h-1}}}.
            \\
             & = \sum_{h=1}^{H'} \left( 4 \sqrt{\frac{L_h^2\log(2KH/\delta)}{L_{h-1}/\abs{\mathcal{S}_{h-1}}}} + \frac{2 L_h(\hat C_{h-1} + C)}{L_{h-1}/\abs{\mathcal{S}_{h-1}}} \right)
            \\
             & \le \sum_{h=1}^{H'} \left( 4\sqrt{2L_hK\log(2KH/\delta)} + 4 K (\hat C_{h-1} + C) \right)
            \\
             & =  4\sqrt{2K \log (2KH/\delta)}\sum_{h=1}^{H'} \sqrt{L_h} + 4K\sum_{h=1}^{H'} (\hat C_{h-1} +  C)
            \\
             & \overset{(a)}\le  4g\sqrt{2KL_0 \log (2KH/\delta) {T}} + 4\sqrt{T}\log_2 T + 4K C\log_2 T.
        \end{split}
    \end{equation}
    where inequality (a) is due to~\eqref{eq:bound-Delta-k} and~\eqref{eq:bound-Delta-k-h-star}, and
    inequality (b) is because \(\sum_{h=1}^{H'} \sqrt{L_h} \le g \sqrt{T}\) since the summation of a geometric increasing series is upper bounded by \(g\sqrt{L_{H'}}\), for some universal constant \(g>0\), and \(L_{H'} \le T\).

    Next, we bound Part II as follows,
    \begin{equation}
        \label{eq:bound-regret-part-II}
        \begin{split}
            \sum_{h=h'+1}^{H'} \sum_{k\in\mathcal{S}_h} \frac{L_h}{\abs{\mathcal{S}_h}} (\mu_{k^*} - \mu_{k_h^*})
             & \overset{(a)}\le \sum_{h=h'+1}^{H'} \sum_{k\in\mathcal{S}_h} \frac{L_h}{\abs{\mathcal{S}_h}} \frac{2CK}{L_0 2^{h'+1}}
            \\
             & = \sum_{h=h'+1}^{H'}  \frac{2CKL_h}{L_0 2^{h'+1}}
            \\
             & = {2CK} \sum_{h=h'+1}^{H'}  2^{h-(h'+1)}
            \\
             & \overset{(b)}\le KC^2,
        \end{split}
    \end{equation}
    where inequality (a) is due to~\eqref{eq:remaining-best-arm-close-to-optimal-arm}, and
    inequality (b) is because \(H' - h' \le H - h' \le \log_2 C\).

    Hence, substituting~\eqref{eq:bound-regret-part-I} and~\eqref{eq:bound-regret-part-II} into the regret bound, we have the pseudo-regret bound with a probability of at least \(1-\delta\) as follows,
    \[
        \begin{split}
            \bar R_T
             & \le L_0 + 4g\sqrt{2K \log (2KH/\delta) {T}} + 4\sqrt{T}\log_2 T + 4K C\log_2 T + KC^2
            \\
             & = {O}\left( \sqrt{KT\log\left( \frac{K\log T}{\delta} \right)} + \sqrt{T}\log T +  KC\log T + KC^2\right).
        \end{split}
    \]
    With a similar analysis as in the proof in Theorem~\ref{thm:addtive-upper-bound-known-attack}, we can show that the realized regret \(R_T\) also has the same upper bound (in order) with a probability of at least \(1-\delta\).
\end{proof}

\end{document}